\documentclass{article}

\usepackage{arxiv}
\usepackage{amsmath}
\usepackage{amssymb}
\usepackage{amsthm}
\usepackage[utf8]{inputenc} 
\usepackage[T1]{fontenc}    
\usepackage{hyperref}       
\usepackage{url}            
\usepackage{booktabs}       
\usepackage{amsfonts}       
\usepackage{nicefrac}       
\usepackage{microtype}      
\usepackage{lipsum}
\usepackage{fancyhdr}       

\usepackage{natbib}

\author{
  Landi He \\
  Shenzhen University of Advanced Technology \\
  \And
  Mingde Yao \\
  CUHK MMLab, CPII under InnoHK \\
  \And
  Shawn Young \\
  Shenzhen University of Advanced Technology \\
  \And
  Lijian Xu  \thanks{Corresponding author} \\
  Shenzhen University of Advanced Technology \\
  \texttt{xulijian@suat-sz.edu.cn}
}

\usepackage{times}
\usepackage{latexsym}

\usepackage{inconsolata}

\usepackage{graphicx}
\usepackage{subcaption}

\usepackage{amssymb}
\usepackage{amsthm}

\usepackage{xcolor}
\usepackage{colortbl}
\definecolor{rowblue}{rgb}{0.90,0.93,0.98}
\definecolor{rowgray}{rgb}{0.93,0.93,0.93}
\newcommand{\venue}[1]{\textit{(#1)}}

\theoremstyle{plain}
\newtheorem{proposition}{Proposition}

\theoremstyle{remark}

\usepackage[capitalize]{cleveref}
\usepackage{makecell}
\usepackage{arydshln}

\usepackage[capitalize]{cleveref}

\usepackage{amsmath}
\usepackage{amssymb}
\usepackage{algorithm}
\usepackage{algpseudocode}

\usepackage{makecell}
\usepackage{arydshln}

\title{DiffPrune: Differentiable Information Throttling for Token Pruning in Vision-Language Models}

\begin{document}
\maketitle

\begin{abstract}
Visual token pruning reduces the computational cost of Vision-Language Models (VLMs) by removing redundant visual tokens. The key is to learn a score that measures whether a token is useful. Existing methods typically rely on Gumbel-Softmax to \textit{approximate} discrete selection during training. Such selectors make the score depend on the behavior of a relaxed pruning operator, not directly on the consequence of information loss. In this paper, we propose \textbf{DiffPrune}, which gives token scores a direct meaning. During training, DiffPrune keeps all tokens and \textit{weakens each token's information} according to its score. If weakening a token hurts the task, the scorer is pushed to protect it; if not, the token can receive a lower score. Because the \textit{loss is differentiated} through this actual information-throttling path, the scorer avoids the unstable surrogate path of relaxed token selection. DiffPrune implements this idea with an \textbf{Information Throttler}, which injects variance-preserving noise into visual tokens, where high-score tokens remain close to their original representations, while low-score tokens carry less original information. At inference, the throttler is removed, and hard top-$K$ pruning is applied using the learned scores. Across ten VLM benchmarks, DiffPrune retains $96.5\%$ of full-model accuracy while accelerating LLM prefill by $2.85\times$, with only $0.69$ ms inference overhead. Code will be publicly available. 
\end{abstract}

\section{Introduction}
\label{sec:intro}

Modern vision-language models (VLMs) incur substantial inference costs because images produce many visual tokens \citep{arefeen2024vita}. Recent work reduces this cost through token compression, adaptive computation, and visual-token pruning \citep{vasu2025fastvlm,dong2025mmtok,feng2026see}. Visual-token pruning directly shortens the input sequence without adding latent representations or auxiliary decoding stages \citep{jeddi2025similarity}. Its effectiveness, however, depends on reliably identifying which tokens to retain.

Token pruning connects training, which needs informative Scorer updates, with deployment, which only requires a ranking under a prescribed budget.
To train a learnable pruner, the discrete keep-or-drop decision must be handled during backpropagation. Gumbel-Softmax has emerged as the standard solution. It relaxes the discrete choice into a continuous mask, enabling gradient flow through a reparameterized categorical distribution, and is often paired with a straight-through estimator (STE) to execute hard sampling while retaining a smooth backward path during training.
The relaxation supplies a standard backpropagation path from task loss to token scores before hard selection is used at deployment.

\begin{figure}[t]
    \centering
    \includegraphics[width=0.90\linewidth]{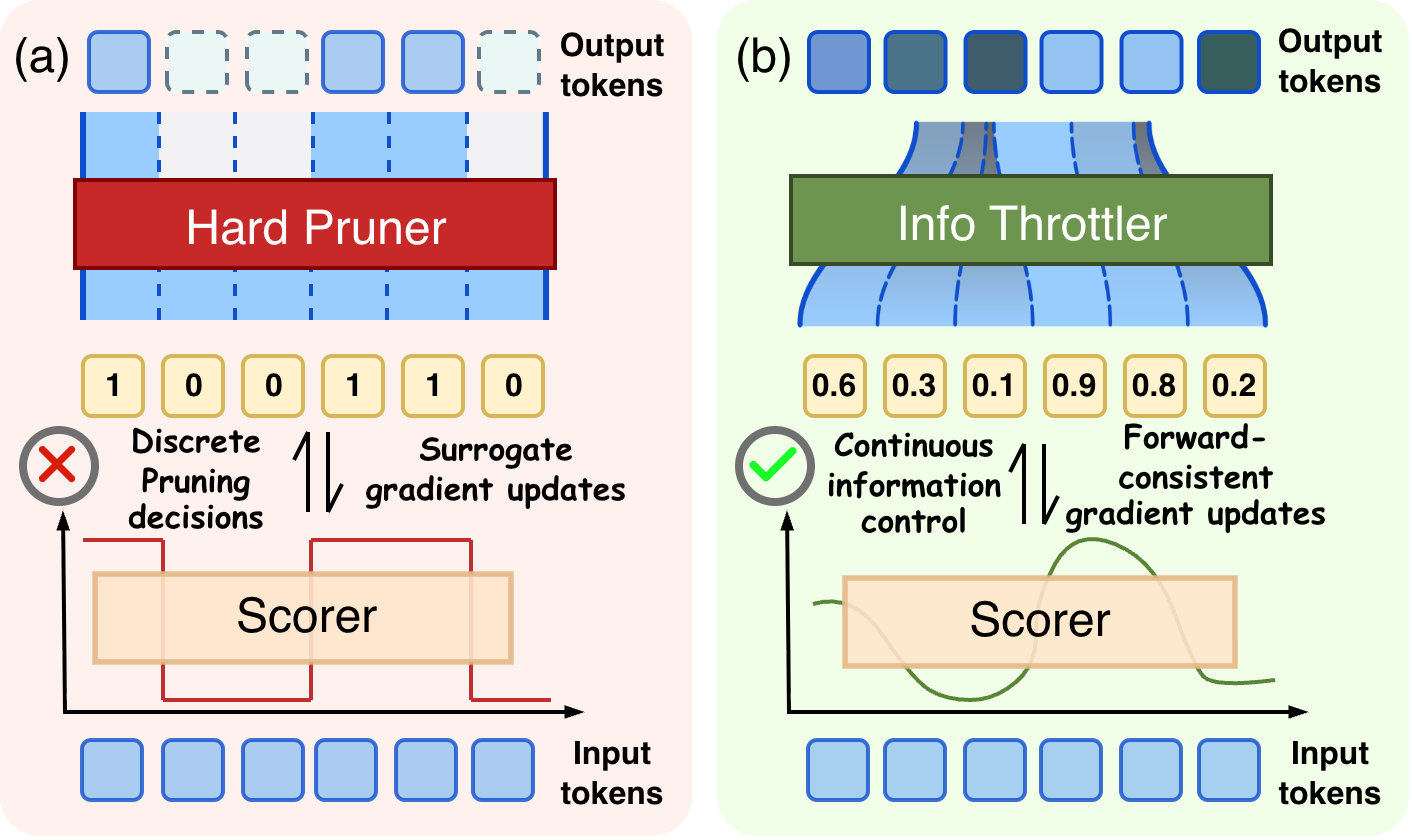}
    \caption{Comparison of training paradigms. (a) Discrete pruning decisions with surrogate gradient updates. (b) Continuous information control with forward-consistent gradient updates.}
    \label{fig:overview}
\end{figure}

In the Gumbel-STE formulation studied here, the issue is not a lack of smoothness in the relaxation itself, but a mismatch between the executed forward operator and the differentiated backward operator. As shown in \Cref{fig:overview}(a), the forward pass makes discrete pruning decisions by executing a hard mask, whereas the scorer is trained with surrogate gradient updates that substitute the Jacobian of a continuous relaxation that did not produce the forward loss \citep{zhao2025fine,liang2025dynamic}. We refer to this as a \emph{forward--backward estimator mismatch}. Consequently, these surrogate gradient updates do not differentiate the loss-producing computation, which can distort the learned token importance and make optimization noisy or unstable, particularly as the backbone grows.

We validate this diagnosis with controlled probes on DeiT variants~\citep{touvron2021training}. As shown in \Cref{fig:gradient-scaling}, DiffPrune achieves $4.4\times$ to $28.4\times$ higher cross-batch gradient-direction consistency than Gumbel-Softmax across model sizes. This consistency gap translates directly into accuracy stability: on DeiT-Base, Gumbel collapses to near-random selection with high variance, while DiffPrune maintains robust performance. The mismatch becomes more consequential as the backbone scales.

\begin{figure}[t]
    \centering
    \includegraphics[width=\linewidth]{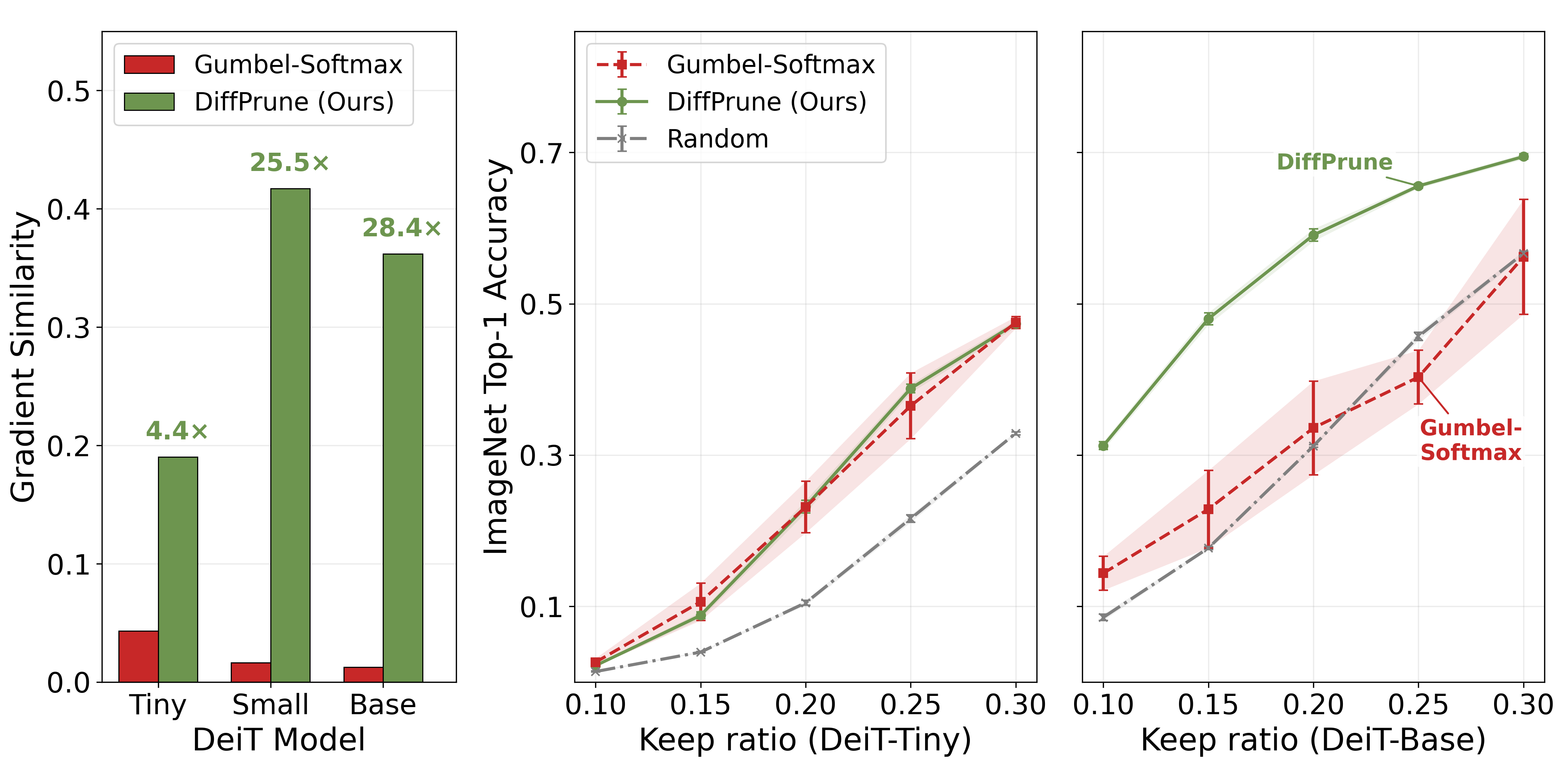}
    \caption{Gradient consistency and accuracy in controlled DeiT probes~\citep{touvron2021training}. Left: cross-batch gradient-direction similarity across DeiT sizes. Middle and right: accuracy at different keep ratios on DeiT-Tiny and DeiT-Base.}
    \label{fig:gradient-scaling}
\end{figure}

Motivated by these findings, we propose DiffPrune, a fully differentiable visual-token pruning framework for VLMs. As shown in \Cref{fig:overview}(b), DiffPrune avoids discrete selection during training and instead controls token information continuously under a fixed budget. Backpropagation differentiates the same throttling operations executed in the training forward pass, providing a continuous learning path for the Scorer.
DiffPrune employs a Scorer that produces one importance logit for each token. A Soft Top-$K$ head \citep{xie2020differentiable} maps these logits to continuous weights whose sum equals the target budget. During training, these weights determine how strongly the corresponding tokens are restricted, without inserting a hard keep-or-drop decision into the training graph.
The Scorer's output then guides an Information Throttler, which injects Gaussian noise into each token with variance modulated by its importance score: lower-scoring tokens receive stronger perturbation, suppressing their contribution to the downstream task. We use square-root interpolation coefficients to approximately preserve the scale of the token representations. For fixed sampled noise, the task loss remains differentiable through the throttling operation executed during training.
At inference, the Throttler is removed and the original top-$K$ tokens are retained according to the learned scores.

Our contributions are:
\begin{itemize}
    \item We show that the limitation of surrogate-gradient pruning lies at the operator level rather than the scorer level, and propose DiffPrune as a fully differentiable framework for visual-token pruning.
    \item We introduce a VP-Noise Gate inside the Information Throttler. It interpolates each token with Gaussian noise according to its continuous score, suppressing token content without relying on a discrete keep-or-drop operation.
    \item In controlled DeiT probes, DiffPrune reaches up to $28.4\times$ higher cross-batch gradient-direction coherence than Gumbel-Softmax; across LLaVA-1.5-7B, LLaVA-NEXT-7B, and Qwen2.5-VL-7B evaluations, it preserves high task performance under aggressive pruning while adding only $0.69$\,ms of inference overhead.
\end{itemize}

\section{Related Work}
\label{sec:related-work}

\subsection{Visual Token Importance Modeling}
\label{sec:related-free}

Many methods derive token-importance signals directly from a pretrained VLM \citep{shao2025survey,yao2026towards}.
Attention-based methods use saliency from the vision tower or language decoder \citep{fastv,sparsevlm,fastervlm,vtw,hired,pyramiddrop,topv}.
Redundancy-based methods remove or merge visually similar features \citep{tome,prumerge,visionzip,fitprune,feather-throttle}.
Coverage-based methods preserve diverse or spatially distributed evidence \citep{divprune,dart,holov,balanced-token-pruning,fangprune,wu2025vlm}.
ICCTP uses instruction-guided cross-modal clustering, while AgilePruner studies attention and diversity for adaptive token pruning \citep{ICCTP,agilepruner}.
Together, these training-free methods define importance using saliency, redundancy, or representativeness and require no selector training for deployment.
These signals emphasize different aspects of ranking: saliency reflects response strength, redundancy reflects similarity, and coverage reflects representation across the input.

Other approaches learn token rankings from supervision or lightweight pretraining.
LearnPruner learns a pruning module before combining its ranking with attention-guided selection inside the LLM \citep{takezoe2026learnpruner}.
OC-VTP pretrains an object-centric pruner with a token-reconstruction objective and then inserts it into existing VLMs without model-specific fine-tuning \citep{li2026ocvtp}.
Together with training-free methods, they span fixed scoring rules, pretrained pruners, and task-conditioned selectors.

\subsection{Optimizing Learned Token Selection}
\label{sec:related-based}

A learned selector must connect continuous importance scores to a discrete subset under a fixed token budget.
A common approach trains this step with a continuous surrogate.
DynamicViT samples token survival with Gumbel-Softmax and differentiates the soft relaxation \citep{rao2021dynamicvit}.
ATP-LLaVA uses a temperature-sharpened sigmoid-threshold mask \citep{atp-llava}, while Dynamic-LLaVA and LightVLA use binary or argmax-like selections with Gumbel-Softmax-style backward paths \citep{huang2025dynamic,jiang2025better}.
These methods use different scorers and pruning designs, but all optimize discrete or nearly discrete forward choices through smooth backward approximations.
Despite different estimators, these methods share an interface in which token-level preferences are converted into a subset that satisfies the pruning budget.

Other methods change the gradient estimator or train discrete selection with reinforcement learning.
Shiva-DiT replaces the softmax surrogate with a residual-aware STE in diffusion transformers \citep{zhang2026shiva}.
TwigVLM++ combines distillation with policy-gradient reinforcement learning \citep{twigvlmpp}, while TOP-RL learns progressive, task-conditioned pruning policies through reinforcement learning \citep{wang2026toprl}.
DiffPrune instead removes discrete selection from training and differentiates the continuous information-throttling operations used in the forward pass.
This operator-level change, analyzed in \Cref{sec:pilot}, distinguishes DiffPrune from methods that retain discrete selection and change how it is optimized.

\section{Motivation}
\label{sec:pilot}

\subsection{Limitations of Surrogate-Gradient Pruning}
\label{sec:pilot-limits}

Training-based token-pruning methods typically share two components.
A learnable scorer assigns importance logits $\mathbf{s}\in\mathbb{R}^{N}$ to the tokens, and a pruning operator retains the top-$K$ tokens (see \Cref{fig:overview}(a)).
The scorer is differentiable, whereas hard top-$K$ selection is not, so the downstream loss cannot be directly backpropagated through the pruning operator.

These methods commonly use Gumbel perturbations together with a straight-through estimator (STE).
Let $\mathbf{g}$ denote sampled Gumbel noise, $H_K$ the hard top-$K$ operator, and $R_\tau$ its continuous relaxation.
The forward and backward computations can be written as
\begin{align}
\mathrm{forward:}\quad
& \mathbf{m}=H_K(\mathbf{s}+\mathbf{g})\in\{0,1\}^{N}, \nonumber\\
\mathrm{backward:}\quad
& \frac{\partial\mathcal{L}}{\partial\mathbf{s}}
\leftarrow
J_{R_\tau}(\mathbf{s}+\mathbf{g})^{\top}
\frac{\partial\mathcal{L}}{\partial\mathbf{m}}.
\label{eq:gs-ste}
\end{align}
Here, $J_{R_\tau}$ is the Jacobian of the relaxation used only in the backward pass.
Thus, the forward pass evaluates a discrete mask, whereas the backward pass differentiates a soft operator that was not executed in the forward computation.
The hard operator $H_K$ is piecewise constant with respect to $\mathbf{s}$: as long as the ordering of the scores does not change, the retained subset remains unchanged and its derivative is almost everywhere zero \citep{shah2024improving}.
Only when a score crosses the $K$-th order statistic and triggers an index swap does the output mask change discontinuously.

This forward--backward mismatch generally makes the STE gradient a biased estimator \citep{shekhovtsov2021bias}:
\begin{equation}
\mathbb{E}_{\mathbf{g}}\!\left[
\widehat{\nabla_{\mathbf{s}}\mathcal{L}}_{\mathrm{STE}}
\right]
\neq
\nabla_{\mathbf{s}}\,
\mathbb{E}_{\mathbf{g}}\!\left[
\mathcal{L}\!\left(H_K(\mathbf{s}+\mathbf{g})\right)
\right].
\label{eq:ste-bias}
\end{equation}
Additionally, hard selection is sensitive near the top-$K$ boundary: small perturbations can swap similarly scored tokens, abruptly changing the retained subset and the downstream gradient.
In \cref{sec:pilot-alt}, we examine the resulting optimization instability and introduce a gradient-continuous alternative.

\subsection{A Gradient-Continuous Alternative}
\label{sec:pilot-alt}

Hard selection makes the training loss sensitive to small changes in token scores.
A continuous training path should instead make small score changes produce gradual changes in the loss.

\begin{figure}[t]
  \centering
  \includegraphics[width=\columnwidth]{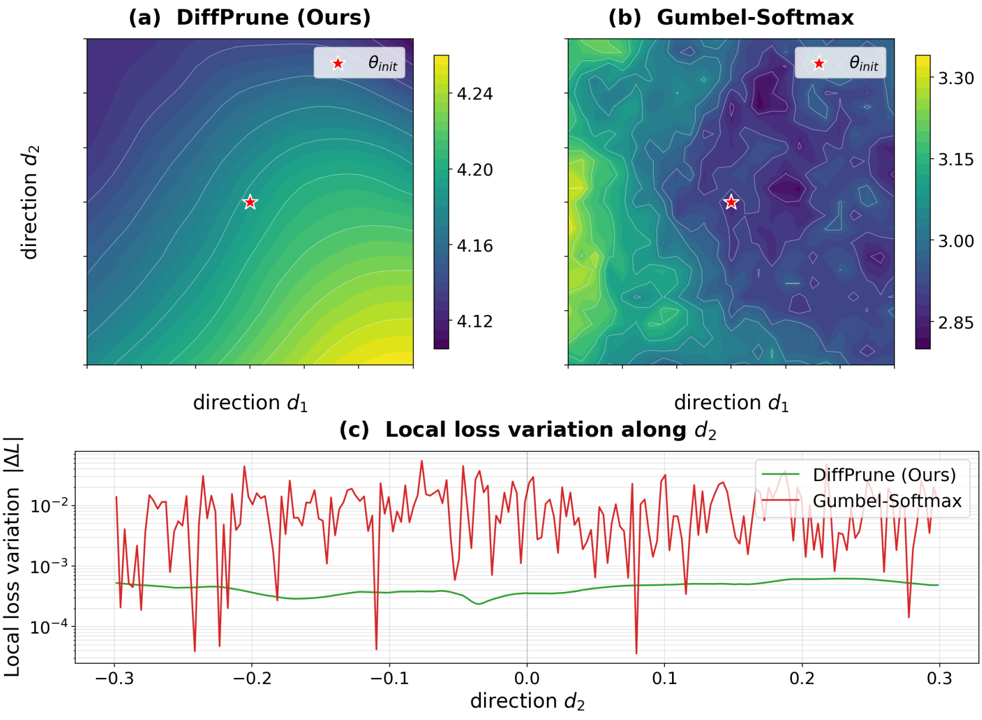}
  \caption{
  Loss around the scorer initialization.
  (a--b) Two-dimensional loss slices for DiffPrune and Gumbel-Softmax.
  (c) Loss variation along one direction.
  }
  \label{fig:pilot}
\end{figure}

\Cref{fig:pilot}(a--b) shows this difference.
The loss changes gradually for DiffPrune, while Gumbel-Softmax has sharp local variations.
With hard top-$K$, a small change in the scores can swap tokens at the selection boundary, so nearby scorer parameters can produce different token subsets and different losses.
DiffPrune avoids this discrete switch during training, so nearby scorer parameters give similar losses.

\Cref{fig:pilot}(c) shows the same pattern along one direction.
The DiffPrune curve changes slowly, while the Gumbel-Softmax curve jumps repeatedly; its mean stepwise variation is roughly $24\times$ larger.
These results support the idea that training should use token importance to control how much information each token carries, rather than turn it immediately into a keep-or-drop decision.
\Cref{sec:method} describes how we apply this principle to visual-token pruning.

\section{Methodology}
\label{sec:method}

\subsection{Overview}
\label{sec:method-arch}

DiffPrune views visual-token pruning as learning how strongly to restrict each token under a fixed budget.
Given an image $\mathbf{I}$, the frozen vision encoder produces visual tokens $\mathbf{X}^v=\mathcal{E}_v(\mathbf{I})\in\mathbb{R}^{N\times d_v}$.
As shown in \cref{fig:framework}, DiffPrune has two components.
The Scorer contains $\mathcal{S}_\theta$ and a budgeted Soft Top-$K$ head $\Phi_K$.
The Information Throttler contains a VP-Noise Gate $\mathcal{G}$ and a train-only Diagonal-Attention Block $\mathcal{D}_\phi$.
The base VLM is frozen; only $\theta$ and $\phi$ are optimized.

Let $\mathbf{W}$ denote the text tokens, $p_\Theta$ the frozen language model, and $\mathbf{Z}$ the multimodal sequence consumed by it.
Training is written as
\begin{equation}
\begin{aligned}
\min_{\theta,\phi}\quad
& -\sum_t \log p_\Theta\!\left(y_t^* \mid y_{<t}^*, \mathbf{Z}\right),\\
\mathbf{Z}
&= [\,\mathbf{P}_{v\to\ell}(\bar{\mathbf{X}}^v);\,\mathcal{E}_t(\mathbf{W})\,],
\end{aligned}
\label{eq:train-fwd}
\end{equation}
where $\bar{\mathbf{X}}^v=\mathcal{D}_\phi(\mathcal{G}(\mathbf{X}^v;\boldsymbol{\alpha}))$ and
$\boldsymbol{\alpha}=\Phi_K(\mathcal{S}_\theta(\mathbf{X}^v)/\tau)$ with $\sum_i\alpha_i=K$.
Here $K$ is the token budget, i.e., the target number of visual tokens retained by the deployed pruner; during training, the token scores satisfy $\sum_i\alpha_i=K$.
Because their sum is fixed, the Scorer cannot assign a high score to every token.
Minimizing the language-model loss trains it to protect tokens whose suppression hurts the task and penalize tokens whose suppression has less effect.
The scores control the restriction imposed by DiffPrune; they do not estimate each token's intrinsic information content.
Hard token selection is absent from \cref{eq:train-fwd}, so the Scorer is trained through the same continuous operations used in the forward pass.

\begin{proposition}[Forward-consistent score gradient]
\label{prop:forward-consistent}
Let $\ell_{\boldsymbol{\epsilon}}(\theta,\phi)$ be the loss in \cref{eq:train-fwd} for fixed data and sampled VP noise $\boldsymbol{\epsilon}$. For finite $\tau$, assume that its executed operators are differentiable. Backpropagation then returns the exact pathwise gradient $\nabla_\theta \ell_{\boldsymbol{\epsilon}}(\theta,\phi)$. In contrast, \cref{eq:gs-ste} substitutes $J_{R_\tau}$ for the Jacobian of the executed $H_K$ and therefore generally is not the gradient of its hard-forward loss.
\end{proposition}
\begin{proof}
For fixed $\boldsymbol{\epsilon}$, the chain rule through the executed maps $\mathcal{S}_\theta$, $\Phi_K$, $\mathcal{G}$, $\mathcal{D}_\phi$, and the frozen downstream model yields the stated gradient. In \cref{eq:gs-ste}, $R_\tau$ is absent from the hard forward graph, so its Jacobian generally does not differentiate that graph's loss.
\end{proof}
The proposition establishes training-gradient faithfulness only; it neither makes hard top-$K$ differentiable nor bounds the hard-deployment risk.

\begin{figure*}[t]
  \centering
  \includegraphics[width=0.99\textwidth]{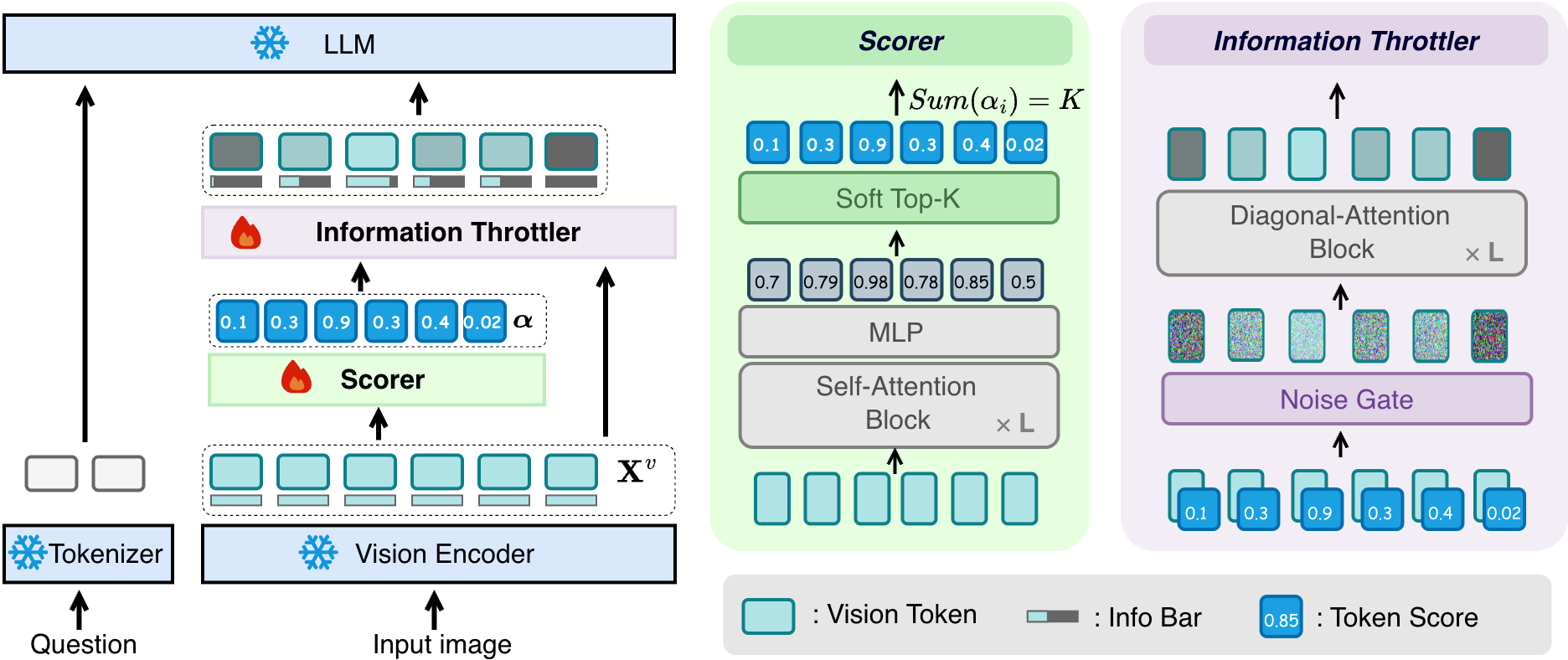}
  \caption{
Overview of DiffPrune. It has two components: a Scorer and an Information Throttler. The Scorer produces token scores whose sum is $K$. The VP-Noise Gate suppresses lower-scoring tokens more strongly, and the train-only Diagonal-Attention Block maps the noised tokens back to the visual representation space expected by the frozen projector before they enter the LLM. At inference, hard top-$K$ retains the original tokens with the highest scores, and the Information Throttler is removed.
}
\label{fig:framework}
\end{figure*}

At inference, DiffPrune uses the trained scorer as a ranking function and gathers the original top-$K$ visual tokens:
\begin{equation}
\hat{\mathbf{X}}^v
= \operatorname{TopK}\!\left(\mathbf{X}^v;\mathcal{S}_\theta(\mathbf{X}^v),K\right)
\in \mathbb{R}^{K\times d_v}.
\label{eq:infer-fwd}
\end{equation}
Here $\operatorname{TopK}(\mathbf{X};\mathbf{s},K)$ denotes score-indexed token gathering.
The retained tokens keep their original order and position indices.
The entire Information Throttler is removed from the deployed graph, leaving one Scorer pass plus this gather operation.

Here, \emph{fully differentiable} refers only to the training-time score-to-loss path; deployment uses hard top-$K$. We distinguish its \emph{forward--backward consistency} from \emph{training--inference operator alignment}. The former follows from \cref{prop:forward-consistent}; the latter asks whether the learned ranking transfers to hard selection and is evaluated separately in \cref{fig:vp-validation,tab:main-llava15,tab:main-llavanext,tab:main-qwen}.

\subsection{Scorer}
\label{sec:method-softtopk}

The Scorer uses a two-layer Transformer encoder and a linear head to assign one logit $s_i$ to each visual token $\mathbf{x}_i^v$, then applies budgeted Soft Top-$K$ to produce continuous token scores $\boldsymbol{\alpha}$ \citep{struski2025lapsum}:
\begin{equation}
\boldsymbol{\alpha}
= \Phi_K\!\left(\mathbf{s}/\tau\right)\in[0,1]^N,
\qquad
\sum_{i=1}^{N}\alpha_i = K.
\label{eq:soft-topk}
\end{equation}
The budget is enforced by $\Phi_K$, so no auxiliary budget loss is used.
Cosine annealing lowers $\tau$ during training, making $\boldsymbol{\alpha}$ approach a binary top-$K$ mask.
Without a score tie at the $K$-th boundary, the low-temperature limit matches the hard ranking in \cref{eq:infer-fwd}.
During training, the downstream VLM receives tokens restricted according to $\boldsymbol{\alpha}$, and backpropagation differentiates the same map used in the forward pass.
This differs from the surrogate-gradient operator in \cref{eq:gs-ste}, where the forward executes a hard mask but the backward follows a different smooth path.

\subsection{Information Throttler}
\label{sec:method-throttler}

The Information Throttler suppresses each token according to its token score.
The VP-Noise Gate adds more noise to tokens with lower scores, but the resulting representations may not be handled reliably by the frozen downstream model.
The train-only Diagonal-Attention Block maps these noised tokens back to the visual representation space expected by the frozen projector before they enter the LLM.

\paragraph{VP-Noise Gate.}
VP-Noise~\citep{ho2020denoising} provides a continuous way to restrict the information carried by each token.
The gate directly accepts a token score $\alpha_i\in[0,1]$ and applies it to token $\mathbf{x}_i$ as
\begin{equation}
\tilde{\mathbf{x}}_i = \sqrt{\alpha_i}\,\mathbf{x}_i + \sqrt{1-\alpha_i}\,\boldsymbol{\epsilon}_i,
\qquad \boldsymbol{\epsilon}_i \sim \mathcal{N}(\mathbf{0},\mathbf{I}),
\label{eq:vp-noise}
\end{equation}
The score $\alpha_i$ controls the strength of the information restriction.
When $\alpha_i$ is close to $1$, the gate preserves most of the original token.
When $\alpha_i$ is close to $0$, the gate replaces most of the token content with noise.
Intermediate scores impose intermediate levels of restriction.
The square-root coefficients keep the feature variance stable.

Once $\boldsymbol{\epsilon}_i$ is sampled for a forward pass, the VP-Noise output remains differentiable with respect to $\alpha_i$.
The language-model loss can therefore train the Scorer through the same operation used in the forward pass.

\paragraph{Diagonal-Attention Block.}\label{sec:method-denoiser}
VP-Noise changes the distribution of visual-token representations.
The frozen downstream model may therefore have difficulty processing them.
We therefore use a train-only Diagonal-Attention Block $\mathcal{D}_\phi$ to map the noised tokens back to the representation space expected by the frozen multimodal projector.
The block uses an identity attention mask,
\[
\mathcal{D}_\phi(\tilde{\mathbf{X}}^v)
= \operatorname{Block}\!\left(\tilde{\mathbf{X}}^v;\operatorname{mask}=\mathbf{I}_N\right),
\]
so each token is transformed independently.
The diagonal attention mask prevents tokens from obtaining information from other positions.
The block can therefore adapt the noised representations without undoing the token-wise information restriction.
It is removed with the Noise Gate at inference.

\begin{figure}[t]
  \centering
  \includegraphics[width=\columnwidth]{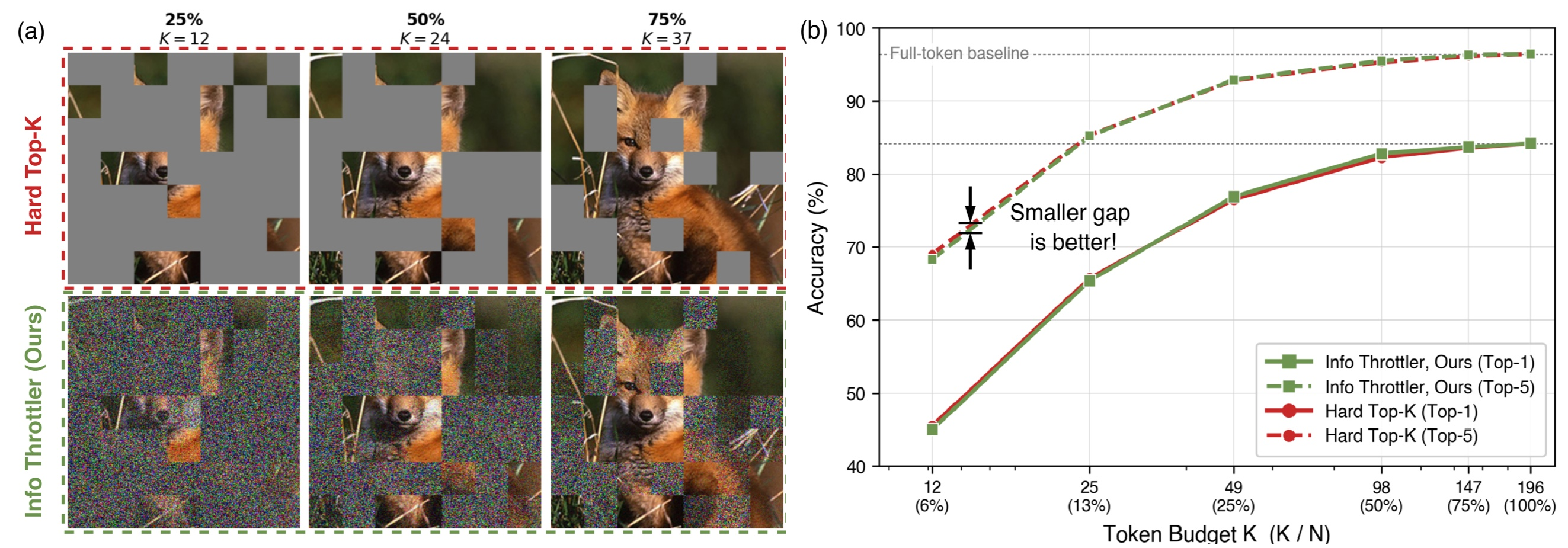}
  \caption{
    Validation of training--inference operator alignment.
    (a) Illustrative examples.
    (b) Top-1 and Top-5 accuracy across token budgets.
    Both operators use identical token scores and budgets from the same trained Scorer; hard top-$K$ is evaluated without retraining or recalibration.
  }
  \label{fig:vp-validation}
\end{figure}

\Cref{fig:vp-validation} directly evaluates whether the ranking learned with the continuous Information Throttler transfers to the hard top-$K$ operator used at deployment.
Both evaluations use identical scores and token budgets from the same trained Scorer, with no retraining or recalibration, so only the restriction operator differs.
\Cref{fig:vp-validation}(a) is an illustration on a $7{\times}7$ grid; \Cref{fig:vp-validation}(b) provides the quantitative comparison.
The small Top-1 and Top-5 gaps across budgets show that the learned ranking transfers to hard selection with limited accuracy degradation.
Because the checkpoint and scores are fixed, this comparison separates the operator change from scorer quality or additional optimization.
The agreement across budgets and both metrics shows that the alignment is not confined to a single operating point.
Together with \cref{tab:main-llava15,tab:main-llavanext,tab:main-qwen}, which use the deployed hard top-$K$ path, \cref{fig:vp-validation} provides a controlled operator-level check alongside downstream VLM evaluation.

The two components give the Scorer a continuous score-to-loss path under a fixed budget.
During training, the Scorer learns which tokens can be suppressed with the least effect on the language-model loss; at inference, hard top-$K$ retains the tokens with the highest learned scores.

\begin{table*}[!t]
  \centering
  {\small
  \setlength{\tabcolsep}{1.6pt}
  \begin{tabular}{@{}lcccccccccccc@{}}
    \toprule
    Method & Loc. & GQA & MMB & MMB-CN & MME & POPE & SQA & VQAv2 & TextVQA & SEED & VizWiz & Avg. \\
    \midrule
    \rowcolor{rowgray} \multicolumn{13}{c}{\textit{Upper Bound, 576 Tokens (100\%)}} \\
    Vanilla & -- & 61.9 & 64.7 & 58.1 & 1862 & 85.9 & 69.5 & 78.5 & 58.2 & 60.5 & 54.3 & 100\% \\
    \midrule
    \rowcolor{rowgray} \multicolumn{13}{c}{\textit{Retain 128 Tokens (77.8\% pruning)}} \\
    FastV \venue{ECCV'24} & L2--32 & 49.6 & 56.1 & 55.9 & 1490 & 59.6 & 60.2 & 61.8 & 50.6 & 55.9 & 51.3 & 85.2\% \\
    VisionZip \venue{CVPR'25} & Pre & 58.9 & 62.6 & -- & 1823 & 83.7 & 68.3 & 76.6 & 56.8 & 55.8 & -- & 96.6\% \\
    HoloV \venue{NeurIPS'25} & Pre & 57.7 & 63.9 & 56.5 & 1802 & 82.8 & 69.8 & 75.5 & 56.8 & -- & 51.5 & 96.8\% \\
    ICCTP \venue{AAAI'26} & Pre & 57.8 & 61.6 & 55.8 & 1775 & 85.2 & 69.0 & 75.6 & 56.8 & -- & 53.0 & 96.7\% \\
    PRUNESID \venue{ICLR'26} & Pre & 58.8 & 62.1 & -- & 1749 & 86.5 & 68.3 & 75.3 & 54.7 & 57.8 & 55.8 & 96.9\% \\
    OC-VTP \venue{CVPR'26} & Pre & 57.5 & 62.0 & -- & 1742 & 84.4 & 69.7 & 74.9 & 56.8 & 60.3 & 55.1 & \underline{97.2\%} \\
    \textbf{DiffPrune (Ours)} & Pre & 57.5 & 62.9 & 57.4 & 1765 & 85.7 & 70.2 & 76.1 & 54.9 & 58.4 & 56.1 & \textbf{97.6\%} \\
    \midrule
    \rowcolor{rowgray} \multicolumn{13}{c}{\textit{Retain 64 Tokens (88.9\% pruning)}} \\
    FastV \venue{ECCV'24} & L2--32 & 46.1 & 48.0 & 52.7 & 1256 & 48.0 & 51.1 & 55.0 & 47.8 & 51.9 & 50.8 & 76.8\% \\
    VisionZip \venue{CVPR'25} & Pre & 57.0 & 61.5 & -- & 1756 & 80.9 & 68.8 & 74.2 & 56.0 & 53.4 & -- & 94.2\% \\
    HoloV \venue{NeurIPS'25} & Pre & 55.3 & 63.3 & 55.1 & 1715 & 80.3 & 69.5 & 72.8 & 55.4 & -- & 52.8 & 94.8\% \\
    ICCTP \venue{AAAI'26} & Pre & 55.6 & 60.3 & 54.6 & 1704 & 82.8 & 68.9 & 73.0 & 55.5 & -- & 52.5 & 94.3\% \\
    PRUNESID \venue{ICLR'26} & Pre & 57.1 & 58.8 & -- & 1733 & 83.8 & 67.8 & 73.7 & 54.2 & 56.1 & 56.9 & \underline{95.1\%} \\
    OC-VTP \venue{CVPR'26} & Pre & 54.6 & 59.3 & -- & 1668 & 80.2 & 69.4 & 73.2 & 55.0 & 56.4 & 55.4 & 94.0\% \\
    \textbf{DiffPrune (Ours)} & Pre & 56.8 & 62.6 & 56.6 & 1723 & 83.4 & 70.1 & 73.9 & 54.3 & 57.6 & 57.2 & \textbf{96.5\%} \\
    \midrule
    \rowcolor{rowgray} \multicolumn{13}{c}{\textit{Retain 32 Tokens (94.4\% pruning)}} \\
    FastV \venue{ECCV'24} & L2--32 & 41.5 & 37.8 & 33.2 & 1090 & 32.5 & 42.6 & 43.4 & 42.5 & -- & -- & 58.6\% \\
    VisionZip \venue{CVPR'25} & Pre & 51.8 & 57.7 & 50.3 & 1536 & 68.7 & 68.8 & 67.1 & 53.1 & -- & 52.9 & 88.3\% \\
    ICCTP \venue{AAAI'26} & Pre & 52.9 & 58.0 & 47.1 & 1602 & 80.4 & 69.4 & 68.0 & 53.3 & -- & 51.9 & 89.9\% \\
    OC-VTP \venue{CVPR'26} & Pre & 55.3 & 56.8 & -- & 1581 & 81.5 & 67.4 & 68.6 & 51.5 & 57.1 & 56.6 & \underline{92.0\%} \\
    \textbf{DiffPrune (Ours)} & Pre & 54.7 & 61.4 & 54.1 & 1655 & 81.2 & 69.9 & 71.7 & 52.4 & 56.5 & 56.9 & \textbf{94.0\%} \\
    \bottomrule
  \end{tabular}
  }
  \caption{
    Results on LLaVA-1.5-7B under three retained-token budgets.
    Loc.\ indicates the pruning location, and Avg.\ reports retention normalized by the unpruned upper bound.
    Best values are bold; second-best values are underlined.
  }
  \label{tab:main-llava15}
\end{table*}

\section{Experiments}
\label{sec:experiments}

\subsection{Experimental Setup}
\label{sec:exp-setup}

DiffPrune is trained on $10\%$ of the ImageNet-1K training split~\citep{deng2009imagenet} with captions from ImageNet-1K-VL-Enriched~\citep{visual_layer_imagenet1k_vl_enriched}.
Only the Scorer and the Diagonal-Attention Block inside the Information Throttler are optimized, adding approximately $84$M trainable parameters, while the base VLM remains frozen.
Complete implementation, training, baseline, profiling, and benchmark details are provided in the supplementary material.

\subsection{Comparative Results}
\label{sec:main-results}

\paragraph{Results on LLaVA-1.5-7B.}
LLaVA-1.5-7B processes $336{\times}336$ inputs that produce $N{=}576$ visual tokens, and \cref{tab:main-llava15} reports three aggressive compression regimes that retain $128/64/32$ tokens (77.8\%/88.9\%/94.4\% pruning).
DiffPrune achieves the strongest average retention at all three budgets.
It exceeds the closest baseline by $0.4$, $1.4$, and $2.0$ percentage points at $K{=}128$, $64$, and $32$, respectively, showing that its advantage persists when only $5.6\%$ of the original visual tokens remain.

\begin{table*}[!t]
  \centering
  {\small
  \setlength{\tabcolsep}{2.5pt}
  \begin{tabular}{@{}lcccccccccc@{}}
    \toprule
    Method & Loc. & GQA & MMB & MMB-CN & MME & POPE & SQA & VQAv2 & TextVQA & Avg. \\
    \midrule
    \rowcolor{rowgray} \multicolumn{11}{c}{\textit{Upper Bound, 2880 Tokens (100\%)}} \\
    Vanilla & - & 64.2 & 67.4 & 60.6 & 1851 & 86.5 & 70.1 & 80.8 & 64.9 & 100\% \\
    \midrule
    \rowcolor{rowgray} \multicolumn{11}{c}{\textit{Retain 320 Tokens (88.9\% pruning)}} \\
    FastV \venue{ECCV'24} & L2--32 & 55.9 & 61.6 & 51.9 & 1661 & 71.7 & 62.8 & 71.9 & 55.7 & 87.6\% \\
    PDrop \venue{CVPR'25} & L8--24 & 56.4 & 63.4 & 56.2 & 1663 & 77.6 & 67.5 & 73.5 & 54.4 & 90.7\% \\
    HoloV \venue{NeurIPS'25} & Pre & 61.7 & 65.3 & 57.5 & 1738 & 83.9 & 68.9 & 79.5 & 58.7 & \underline{95.7\%} \\
    PRUNESID \venue{ICLR'26} & Pre & 60.5 & 63.0 & -- & 1754 & 83.1 & 67.3 & 76.6 & -- & 94.9\% \\
    OC-VTP \venue{CVPR'26} & Pre & 58.1 & 61.7 & -- & -- & 82.3 & 67.4 & -- & 58.2 & 92.6\% \\
    \textbf{DiffPrune} & Pre & 62.3 & 64.7 & 57.4 & 1723 & 85.9 & 72.7 & 78.6 & 56.7 & \textbf{96.1\%} \\
    \bottomrule
  \end{tabular}
  }
  \caption{
    Results on LLaVA-NEXT-7B at $K{=}320$ (88.9\% pruning).
    This high-resolution setting starts from $2880$ visual tokens; Avg.\ reports normalized average retention.
    Best values are bold; second-best values are underlined.
  }
  \label{tab:main-llavanext}
\end{table*}

\paragraph{Results on LLaVA-NEXT-7B.}
LLaVA-NEXT-7B uses $672{\times}672$ inputs and produces $N{=}2880$ visual tokens, five times the token count of LLaVA-1.5-7B.
This setting gives a stronger test of visual-token pruning because the language-model prefill cost grows with the visual prefix length.
\Cref{tab:main-llavanext} reports results at $K{=}320$, where each method keeps only $11.1\%$ of the visual tokens.
DiffPrune obtains $96.1\%$ average retention, the best result among the reported methods and $0.4\%$ higher than HoloV.
We use the same scorer and pruning recipe as in LLaVA-1.5-7B, with no resolution-specific scheduling or architectural change.

\paragraph{Results on Qwen2.5-VL-7B.}
Qwen2.5-VL-7B differs from the LLaVA family in vision encoder, projector and language backbone, and processes images at native resolution, so the visual-token count varies per image and results are reported by pruning rate rather than a fixed budget $K$.
Under the three pruning rates listed in \cref{tab:main-qwen}, DiffPrune achieves the highest average retention at every rate and leads on $13$ of the $15$ benchmark--rate combinations.
At the $77.8\%$ pruning rate, it leads on all five benchmarks; its average-retention margin over HoloV widens from $4.0$ points at $66.7\%$ pruning to $6.1$ points at $88.9\%$ pruning.
No architectural or hyperparameter adaptation is applied across backbones; the same Scorer, Information Throttler, and schedule used on LLaVA-1.5-7B are reused without modification, supporting the claim that the framework is backbone-agnostic by construction.
We further examine this behavior across Qwen2.5-VL model scales in the supplementary material.

\begin{table}[!tb]
  \centering
  %
  {\small
  \setlength{\tabcolsep}{1.5pt}
  \begin{tabular}{@{}lcccccc@{}}
    \toprule
    Method & MMB & MME & POPE & SQA & TextVQA & Avg. \\
    \midrule
    \rowcolor{rowgray} \multicolumn{7}{c}{\textit{Upper Bound (100\%)}} \\
    Vanilla & 82.8 & 2304 & 86.1 & 84.7 & 84.8 & 100\% \\
    \midrule
    \rowcolor{rowgray} \multicolumn{7}{c}{\textit{Token Pruning Rate = 66.7\%}} \\
    FastV \venue{ECCV'24} & 75.7 & 2072 & 82.2 & 78.5 & 77.9 & 92.3\% \\
    HoloV \venue{NeurIPS'25} & 78.3 & 2093 & 85.0 & 79.8 & 78.9 & \underline{94.3\%} \\
    \textbf{DiffPrune} & 81.7 & 2279 & 84.9 & 86.4 & 79.0 & \textbf{98.3\%} \\
    \midrule
    \rowcolor{rowgray} \multicolumn{7}{c}{\textit{Token Pruning Rate = 77.8\%}} \\
    FastV \venue{ECCV'24} & 74.9 & 2036 & 80.7 & 78.0 & 69.0 & 89.2\% \\
    HoloV \venue{NeurIPS'25} & 76.5 & 2043 & 82.3 & 79.8 & 70.3 & \underline{90.8\%} \\
    \textbf{DiffPrune} & 81.0 & 2218 & 82.8 & 83.7 & 76.7 & \textbf{95.9\%} \\
    \midrule
    \rowcolor{rowgray} \multicolumn{7}{c}{\textit{Token Pruning Rate = 88.9\%}} \\
    FastV \venue{ECCV'24} & 69.2 & 1940 & 78.6 & 77.4 & 60.3 & 84.3\% \\
    HoloV \venue{NeurIPS'25} & 72.4 & 2006 & 80.7 & 79.5 & 61.8 & \underline{87.0\%} \\
    \textbf{DiffPrune} & 76.7 & 2113 & 79.8 & 82.9 & 76.7 & \textbf{93.1\%} \\
    \bottomrule
  \end{tabular}
  }
  \caption{
    Results on Qwen2.5-VL-7B across three pruning rates.
    Native-resolution inputs yield image-dependent token counts; Avg.\ is normalized by the unpruned model.
    Best values are bold; second-best values are underlined.
  }
  \label{tab:main-qwen}
\end{table}

\subsection{Efficiency Analysis}
\label{sec:efficiency}

We profile all methods on LLaVA-1.5-7B with $336{\times}336$ inputs and $N{=}576$ visual tokens, using a single NVIDIA A6000 GPU, batch size one, and FP16 precision.
Latencies are averaged over $30$ forward passes after warm-up.
\Cref{tab:efficiency} reports TTFT at $K{=}64$, decomposed into vision encoding, pruning-module decision, and LLM prefill.

\begin{table}[!tb]
  \centering
  {\small
  \setlength{\tabcolsep}{0.8pt}
  \begin{tabular}{@{}lccccccc@{}}
    \toprule
    Method & Loc. & $K$ & FLOPs & Enc. & Pruner & Prefill & TTFT \\
    \midrule
    Full (LLaVA) & ---  & 576 & 8.89 & 29.71 & 0.00 & 118.66 & 149.51 \\
    \midrule
    PruneSID & Pre & 64 & 2.13 & 31.80 & 43.39 & 39.74 & 115.84 \\
    PyramidDrop & L8--24 & 64 & 4.89 & 29.54 & 5.02 & 70.08 & 105.90 \\
    HoloV & Pre & 64 & 2.15 & 31.96 & 2.77 & 41.48 & 77.39 \\
    \textbf{DiffPrune} & Pre & 64 & 2.13 & 30.41 & \textbf{0.69} & 41.61 & \textbf{73.73} \\
    \bottomrule
  \end{tabular}
  }
  \caption{
    Efficiency breakdown on LLaVA-1.5-7B at $K{=}64$.
    Latencies are in ms; FLOPs are in T.
  }
  \label{tab:efficiency}
\end{table}

Vision encoding takes about $30$\,ms for all methods, since pruning is applied after the encoder.
The main differences come from the pruner and the LLM prefill.
PyramidDrop prunes inside the language model, so early layers still process the full visual prefix.
Its LLM prefill remains $70.08$\,ms with $4.89$\,T FLOPs.
In contrast, Pre-LLM methods reduce the visual prefix before it enters the language model, bringing LLM prefill to about $40$\,ms and FLOPs to $2.13$--$2.15$\,T.
Among these methods, the pruner itself becomes the main source of latency.
PRUNESID spends $43.39$\,ms on token selection, over $60\times$ the cost of DiffPrune, while HoloV spends $2.77$\,ms.
DiffPrune needs only $0.69$\,ms for selection and reaches the lowest TTFT in the comparison ($73.73$\,ms).
This follows directly from its inference path in \cref{eq:infer-fwd}: the Information Throttler is removed, leaving only one Scorer pass and index gathering.

\subsection{Ablation Studies}
\label{sec:ablations}

We isolate three training-time design choices in DiffPrune: the continuous score-to-loss path, the VP-Noise Gate, and the Diagonal-Attention Block.
All variants keep the base VLM, Scorer architecture, training data, and schedule fixed.
\Cref{tab:ablation} reports average retention over GQA, MMBench, POPE, and MME on LLaVA-1.5-7B.

\begin{table}[!tb]
  \centering
  {\small
  \setlength{\tabcolsep}{2.0pt}
  \begin{tabular}{@{}llcc@{}}
    \toprule
    Variant & Replaced component & $K{=}64$ & $K{=}128$ \\
    \midrule
    Full DiffPrune & -- & \textbf{95.5}\% & \textbf{97.3}\% \\
    \midrule
    Gumbel-STE & Score-to-loss path & 87.4\% & 91.3\% \\
    Scale gate & VP-Noise Gate & 93.2\% & 94.8\% \\
    Global attention & Diagonal-Attn.\ Block & 92.7\% & 95.0\% \\
    \bottomrule
  \end{tabular}
  }
  \caption{
    Component ablation of DiffPrune on LLaVA-1.5-7B.
    Scores are average retention over GQA, MMBench, POPE, and MME.
  }
  \label{tab:ablation}
\end{table}

Replacing the continuous score-to-loss path with a Gumbel-Softmax + STE variant gives the largest drop.
Retention decreases by $8.1$ and $6.0$ points at $K{=}64$ and $K{=}128$, respectively.
Since the rest of the method is unchanged, this result supports the claim from \cref{sec:pilot} that the gradient-continuous training path is important in the full VLM setting, not only in the DeiT probe.

The other two rows test the internal design of the Information Throttler.
Replacing VP-Noise with the fully differentiable scale gate $\tilde{\mathbf{x}}_i=\alpha_i\mathbf{x}_i$ reduces retention by $2.3$ and $2.5$ points, showing that noise injection gives a stronger information restriction than simple rescaling.
Replacing the Diagonal-Attention Block with a global-attention block reduces retention by $2.8$ and $2.3$ points.
This supports per-token adjustment, since global attention can let corrupted low-score tokens obtain information from high-score tokens.
The supplementary material provides a qualitative analysis of the learned token rankings and their feature-space redundancy.

\section{Conclusion}
\label{sec:conclusion}
We identify the discrete selection operator as a central obstacle in learned visual-token pruning: continuous importance scores must eventually become hard keep-or-drop decisions, and common surrogate-gradient solutions train through a backward path that is not the function executed in the forward pass.
DiffPrune addresses this mismatch by replacing training-time selection with continuous token-information control.
Its Soft Top-$K$ head and Information Throttler keep the scorer-to-loss path analytically differentiable during training, while inference collapses to a hard top-$K$ scorer-only path.
The resulting framework improves gradient-direction coherence by up to $28.4\times$, reduces local loss variation by roughly $24\times$ in our diagnostic study, and adds only $0.69$\,ms of selection overhead on LLaVA-1.5-7B.
Across these settings, the same learned ranking connects the continuous training objective to the hard deployment path, while the Information Throttler remains confined to training.
These findings suggest that full differentiability is a useful design rule for learned pruning, not merely an implementation detail.
Extending the same principle to other sparsity operators remains future work.

\clearpage

\bibliographystyle{unsrt}  
\bibliography{reference}

\clearpage

\appendix

\renewcommand{\thetable}{A\arabic{table}}
\setcounter{table}{0}
\renewcommand{\thefigure}{A\arabic{figure}}
\setcounter{figure}{0}

\section{Implementation Details and Training Hyperparameters}
\label{sec:appendix-impl}

\paragraph{Implementation.}
All experiments are implemented in PyTorch and run on a single NVIDIA RTX A6000 GPU, except that the Qwen2.5-VL-32B model is trained using four GPUs.
The implementation is built on the LLaVA-v1.5 codebase~\citep{liu2023visual_instruction_tuning}.
For every evaluated backbone, the vision encoder, multimodal projector, and language model remain frozen; only the DiffPrune Scorer and the Diagonal-Attention Block inside the Information Throttler are optimized.

\paragraph{Computing environment.}
Experiments are conducted on Ubuntu 22.04.3 LTS with two Intel Xeon Platinum 8352V CPUs, 256 GB of system memory, and one NVIDIA RTX A6000 GPU with 48 GB of memory.
The NVIDIA driver is version 580.105.08, and the PyTorch CUDA runtime is 11.8.
The software environment uses PyTorch 2.0.1, Transformers 4.37.2, Accelerate 0.21.0, Datasets 2.16.1, and Tokenizers 0.15.2.

\paragraph{Architecture.}
The Scorer contains two pre-norm Transformer blocks whose configurations follow the visual-encoder blocks of the corresponding backbone.
A linear head assigns one scalar importance logit to each visual token.
The Information Throttler uses the parameter-free VP-Noise Gate followed by one trainable Transformer block that maps noised tokens back to the representation space expected by the frozen multimodal projector.
The block uses the diagonal attention mask described in the main paper.
For LLaVA-1.5-7B, the Scorer and this trainable block add approximately $84$M parameters, about $1.2\%$ of the backbone model.
At inference time, the Information Throttler is removed; the deployed graph contains only the Scorer and the hard top-$K$ gathering operation defined in the main paper and further characterized in \cref{sec:appendix-softtopk}.

\paragraph{Training.}
The Scorer and the Diagonal-Attention Block inside the Information Throttler are trained jointly in a single stage on images and enriched English captions from the training split of ImageNet-1K-VL-Enriched~\citep{visual_layer_imagenet1k_vl_enriched}.
We construct one fixed $10\%$ subset by random sampling without replacement using seed $42$ and reuse it across all runs.
We use AdamW with learning rate $2{\times}10^{-4}$ and cosine decay.
Training uses a per-GPU batch size of $4$, four gradient-accumulation steps, an effective global batch size of $16$, and BF16 mixed precision.
The objective is caption generation under the next-token negative log-likelihood defined in the main paper; no labels from downstream evaluation benchmarks are used for training.
The Soft Top-$K$ temperature follows the cosine schedule specified in \cref{sec:appendix-softtopk}, gradually sharpening the token scores toward the hard top-$K$ ranking used at inference.
Training runs for at most one epoch over the subset, with early stopping determined by validation loss, and takes approximately ten GPU-hours.

\paragraph{Randomness and reporting.}
Our downstream VLM, ablation, and scalability results are obtained from a single run with global random seed $42$.
The controlled DeiT experiments shown in Figure~2 of the main paper use five random seeds $\{42,123,456,789,1000\}$ and report the mean and one standard deviation.
For our experiments, the seed is applied to Python's random module, NumPy, the PyTorch CPU and CUDA random number generators, and DataLoader workers.

\paragraph{Code availability.}
Source code is not provided during review.
It will be released under the Apache-2.0 license upon publication.

\paragraph{Backbones and token budgets.}
LLaVA-1.5-7B uses $336{\times}336$ images and produces $N{=}576$ visual tokens.
We evaluate $K{\in}\{128,64,32\}$, corresponding to $77.8\%$, $88.9\%$, and $94.4\%$ pruning.
LLaVA-NEXT-7B uses $672{\times}672$ inputs and produces $N{=}2880$ visual tokens; we report $K{=}320$.
Qwen2.5-VL-7B processes images at native resolution, so the number of visual tokens varies by example; for this backbone, results are reported by pruning rate rather than by a fixed $K$.

\paragraph{Baselines and reporting convention.}
Baselines include FastV~\citep{fastv}, SparseVLM~\citep{sparsevlm}, DART~\citep{dart}, PyramidDrop~\citep{pyramiddrop}, DivPrune~\citep{divprune}, VisionZip~\citep{visionzip}, HoloV~\citep{holov}, PRUNESID~\citep{fangprune}, OC-VTP~\citep{li2026ocvtp}, ICCTP~\citep{ICCTP}, AgilePruner~\citep{agilepruner}, ATP-LLaVA~\citep{atp-llava}, Dynamic-LLaVA~\citep{huang2025dynamic}, p-MoD~\citep{zhang2025pmod}, and GlimpsePrune~\citep{zeng2025glimpse}.
Baseline scores are taken directly from the corresponding papers rather than reproduced in our environment.
Because prior work does not always report every benchmark under every token budget, unavailable results are marked as ``--''.
When applicable, Loc.\ records whether pruning happens before the language model (Pre) or inside it (L$a$--L$b$).

\paragraph{Efficiency profiling.}
Latency is profiled on LLaVA-1.5-7B with $336{\times}336$ inputs, $N{=}576$ visual tokens, batch size one, and FP16 precision on the same NVIDIA A6000 GPU.
We average $30$ forward passes after warm-up.
Time-to-first-token is decomposed into vision encoding, pruning-module decision time, and language-model prefill, so the reported pruning overhead isolates the cost added by each selector.

\section{Soft Top-$K$ Operator and Low-Temperature Limit}
\label{sec:appendix-softtopk}

\paragraph{Operator framework.}
Let $\mathbf{s}\in\mathbb{R}^{N}$ be the raw importance logits, let $K\in\{1,\ldots,N-1\}$ be the token budget, and let $\tau>0$ be the temperature.
At the level needed for analysis, a threshold-adjusted Soft Top-$K$ operator can be represented as
\begin{equation}
  \begin{aligned}
    \alpha_i^{(\tau)}
    &= \left[\Phi_K(\mathbf{s}/\tau)\right]_i
     = f\!\left(\frac{s_i-\lambda_\tau(\mathbf{s},K)}{\tau}\right),\\
    \sum_{i=1}^{N}\alpha_i^{(\tau)} &= K,
  \end{aligned}
  \label{eq:appendix-softtopk-framework}
\end{equation}
where $f:\mathbb{R}\to(0,1)$ is smooth and strictly increasing, with $f(-\infty)=0$ and $f(+\infty)=1$.
The shared threshold $\lambda_\tau(\mathbf{s},K)$ adapts to the logits and the target budget.
The exact response function and threshold solver are implementation-specific.
Algorithm~\ref{alg:appendix-softtopk} summarizes the computation through an abstract \textsc{BudgetThreshold} interface, which denotes any differentiable realization that enforces the prescribed budget.
LapSum is one representative differentiable construction with this general structure~\citep{struski2025lapsum}.

\begin{algorithm}[H]
\caption{Abstract computation of budgeted Soft Top-$K$ weights}
\label{alg:appendix-softtopk}
\textbf{Input}: importance logits $\mathbf{S}\in\mathbb{R}^{B\times N}$, token budget $K$, temperature $\tau>0$\\
\textbf{Output}: weights $\boldsymbol{\alpha}\in(0,1)^{B\times N}$
\begin{algorithmic}[1]
\For{each batch element $b\in\{1,\ldots,B\}$}
    \State $\mathbf{u}_b \gets \mathbf{S}_b/\tau$
    \State $\lambda_b \gets \textsc{BudgetThreshold}(\mathbf{u}_b,K)$
    \State \textbf{such that} $\sum_{i=1}^{N} f(u_{b,i}-\lambda_b)=K$
    \For{$i=1,\ldots,N$}
        \State $\alpha_{b,i} \gets f(u_{b,i}-\lambda_b)$
    \EndFor
\EndFor
\State \textbf{return} $\boldsymbol{\alpha}$
\end{algorithmic}
\end{algorithm}

\paragraph{Implementation-independent properties.}
Equation~\eqref{eq:appendix-softtopk-framework} places the output in the budget set
\[
  \Delta_K
  = \left\{\boldsymbol{\alpha}\in[0,1]^N:
  \sum_{i=1}^{N}\alpha_i=K\right\},
\]
which enforces the token budget without an auxiliary budget loss.
Because every token shares the same threshold and $f$ is strictly increasing, $s_i>s_j$ implies $\alpha_i^{(\tau)}>\alpha_j^{(\tau)}$; hence the operator preserves the ranking induced by the Scorer.
It is also invariant to a common logit shift:
\[
  \Phi_K((\mathbf{s}+c\mathbf{1})/\tau)
  = \Phi_K(\mathbf{s}/\tau),
  \qquad c\in\mathbb{R},
\]
because the shift is absorbed by the shared threshold.
For finite $\tau$, differentiating through the smooth response and the data-dependent threshold provides the continuous score-to-loss path used during training.

\paragraph{Low-temperature limit.}
Let $s_{\pi_1}\geq\cdots\geq s_{\pi_N}$ denote the logits in descending order, and assume a strict boundary gap $s_{\pi_K}>s_{\pi_{K+1}}$.
Then the Soft Top-$K$ weights approach the hard mask:
\begin{equation}
  \lim_{\tau\to0^+}\alpha_{\pi_j}^{(\tau)}
  =
  \begin{cases}
    1, & j\leq K,\\
    0, & j>K,
  \end{cases}
  \qquad
  \lim_{\tau\to0^+}\Phi_K(\mathbf{s}/\tau)
  = H_K(\mathbf{s}).
  \label{eq:appendix-hardtopk-limit}
\end{equation}
As $\tau$ decreases, it amplifies the separation across the top-$K$ boundary, while the adaptive threshold maintains the exact sum $K$.
The saturation of $f$ therefore sends the highest $K$ weights to one and the remaining weights to zero.
This is an asymptotic statement rather than an equality at any finite temperature.
If $s_{\pi_K}=s_{\pi_{K+1}}$, the hard top-$K$ subset is itself non-unique and requires an additional tie-breaking convention.
Moreover, the limit establishes agreement at the ranking and mask level; it does not assert equality between the task loss under the continuous Information Throttler and that under hard token removal.

\paragraph{Training and inference.}
We use the cosine schedule
\begin{equation}
  \tau_t
  = \tau_{\mathrm{end}}
  + \frac{\tau_{\mathrm{start}}-\tau_{\mathrm{end}}}{2}
  \left[1+\cos\!\left(
  \pi\frac{\min(t,T_{\mathrm{anneal}})}{T_{\mathrm{anneal}}}
  \right)\right],
  \label{eq:appendix-temperature}
\end{equation}
with $\tau_{\mathrm{start}}{=}2.0$, $\tau_{\mathrm{end}}{=}0.1$, and $T_{\mathrm{anneal}}{=}2{,}000$ steps; the temperature remains at $0.1$ thereafter.
Training uses the finite-temperature weights to control the Information Throttler.
At inference, the Throttler is removed and hard top-$K$ is applied directly to the raw logits, without setting $\tau$ to zero or evaluating the abstract threshold representation in \cref{eq:appendix-softtopk-framework}.

\begin{figure*}[t]
  \centering
  \includegraphics[width=\textwidth]{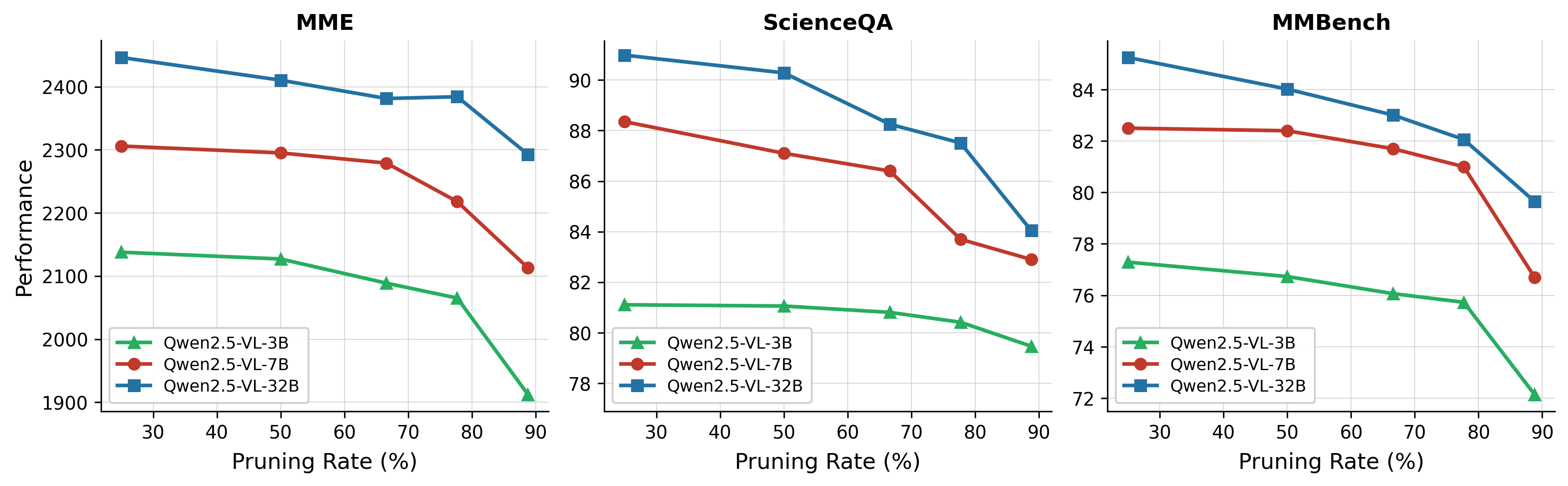}
  \caption{
    Scalability across Qwen2.5-VL model scales.
    Performance on MME, ScienceQA, and MMBench as a function of the visual-token pruning rate for Qwen2.5-VL-3B, 7B, and 32B.
  }
  \label{fig:appendix-scalability}
\end{figure*}

\section{Benchmark Descriptions}
\label{sec:appendix-benchmarks}

We evaluate DiffPrune on ten standard VLM benchmarks spanning compositional reasoning, perception, hallucination, text-centric understanding, and real-world VQA.
We follow the task-specific evaluation protocols and scripts recommended by the official LLaVA repository, use greedy decoding, and use the corresponding official evaluation toolkit or server whenever applicable.

\paragraph{GQA~\citep{hudson2019gqa}.}
GQA is a compositional visual question answering benchmark built from Visual Genome scene graphs.
It emphasizes multi-step reasoning over objects, attributes, spatial relations, and logical operations, making it useful for testing whether token pruning preserves structured visual evidence.

\paragraph{MMBench and MMBench-CN~\citep{liu2024mmbench}.}
MMBench evaluates multimodal perception and reasoning through multiple-choice questions organized into fine-grained ability dimensions.
We report both the English benchmark and its Chinese counterpart, MMBench-CN, which follows the same evaluation taxonomy and enables testing under a language distribution shift.

\paragraph{MME~\citep{fu2023mme}.}
MME evaluates perception and cognition through manually designed yes/no questions.
Its perception subset covers skills such as existence, count, position, and color recognition, while its cognition subset covers broader reasoning categories.
Because each image is paired with controlled questions, MME is sensitive to both visual information loss and hallucinated content.

\paragraph{POPE~\citep{pope}.}
POPE measures object hallucination by querying whether specific objects are present in an image.
Questions are sampled under random, popular, and adversarial strategies, probing whether a model relies on visual evidence rather than dataset-level object co-occurrence priors.

\paragraph{ScienceQA-IMG~\citep{lu2022learn}.}
ScienceQA-IMG contains image-grounded science questions that require both visual understanding and domain knowledge.
We use the image-containing subset, following standard VLM evaluation practice.

\paragraph{VQAv2~\citep{goyal2017making}.}
VQAv2 is an open-ended visual question answering benchmark designed to reduce language priors through complementary image pairs.
It tests whether the model's answer changes with the visual evidence rather than being driven only by question statistics.

\paragraph{TextVQA~\citep{singh2019towards}.}
TextVQA focuses on reading and reasoning over text appearing inside images, such as signs, labels, documents, and scene text.
This benchmark stresses OCR-sensitive visual tokens and is therefore a stringent test for pruning methods that remove large fractions of the visual prefix.

\paragraph{SEED-Bench~\citep{li2023seed}.}
SEED-Bench evaluates image-centric multimodal comprehension across dimensions such as scene understanding, spatial relation, instance identity, and instance attribute recognition.
We use it to assess whether pruning preserves broad semantic coverage beyond conventional VQA accuracy.

\paragraph{VizWiz~\citep{bigham2010vizwiz}.}
VizWiz contains questions about images taken by blind users in everyday settings.
The images often contain unusual framing, low quality, or unanswerable queries, making the benchmark a realistic stress test for robustness under noisy visual inputs.

For each benchmark, we report its standard score.

\section{Scalability Across Model Scales}
\label{sec:appendix-scalability}

To further test whether DiffPrune depends on a particular model capacity, we evaluate the 3B, 7B, and 32B variants of Qwen2.5-VL~\citep{bai2025qwen25vltechnicalreport} on MME, ScienceQA, and MMBench.
We use five visual-token pruning rates: $25\%$, $50\%$, $66.7\%$, $77.8\%$, and $88.9\%$.
The same pruning formulation and training recipe are used across model sizes.

\Cref{fig:appendix-scalability} shows that, across all three scales, performance decreases smoothly as the pruning rate increases, with no abrupt accuracy collapse even at $88.9\%$ pruning.
On MME, the 32B model remains close to $2{,}400$ points at the $77.8\%$ pruning setting, while the 7B and 3B variants show steady rather than discontinuous degradation as the pruning rate increases across the evaluated settings.
MMBench and ScienceQA show the same pattern: the 32B and 7B results remain close through high pruning rates, and all models preserve usable performance under aggressive compression.
These results suggest that DiffPrune's budgeted, score-driven formulation is not tied to a single Qwen2.5-VL model size.

\section{Qualitative Analysis}
\label{sec:visualization}

\begin{figure}[t]
  \centering
  \includegraphics[width=\linewidth]{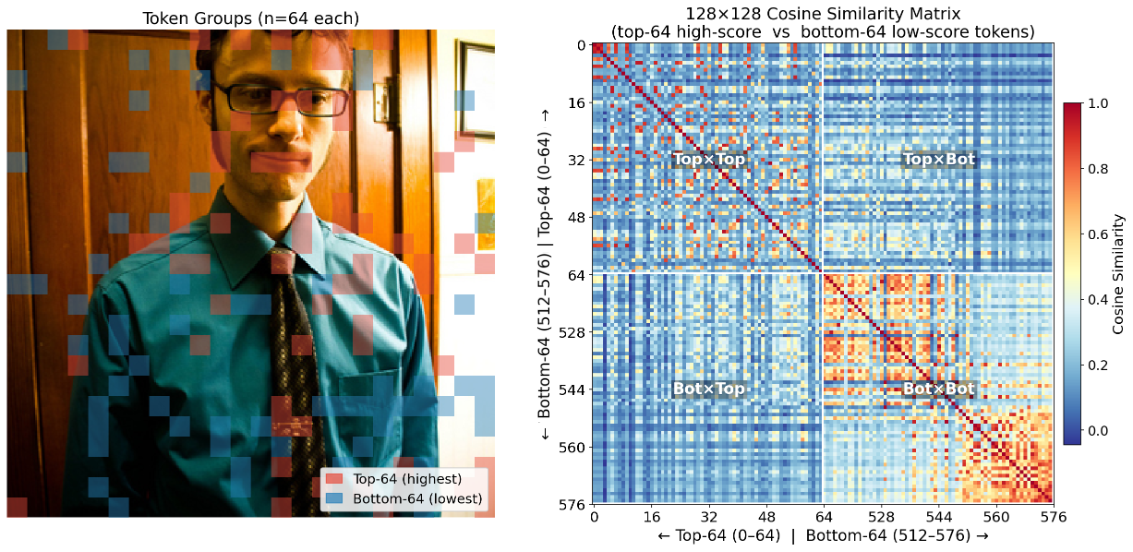}
  \caption{
    Learned token ranking on a VQAv2 example at $K{=}64$.
    Left: Tokens are ranked by learned importance scores, where red and blue indicate the top-$K$ retained tokens and the lowest-scored pruned tokens, respectively.
    Right: Similarity matrix of retained and pruned tokens.
    Pruned tokens are more redundant.
  }
  \label{fig:vqaviz}
\end{figure}

\Cref{fig:vqaviz} visualizes the token scores learned by DiffPrune on an example from the VQAv2 benchmark used in our evaluation.
For readability, we show the two extremes of the score ranking: the top-$64$ tokens cover visually salient regions such as the face, hands, and clothing, while the bottom-$64$ tokens mainly fall on visually uniform background areas.
The accompanying $128{\times}128$ cosine-similarity matrix $[\text{top-}64;|;\text{bottom-}64]$ provides a feature-space view of the same ranking: in this example, the retained-token block has lower internal similarity, whereas the lowest-scored block is more redundant.
This qualitative evidence is consistent with the quantitative results, while remaining illustrative rather than conclusive: the learned Scorer appears to allocate its limited budget toward visually informative and less redundant tokens.

\section{Limitation and Future Work}

DiffPrune relies on the continuity of visual-token embeddings.
VP-Noise throttles a token by interpolating its feature vector with isotropic Gaussian noise under a variance-preserving schedule before the frozen language model consumes it.
This assumption is natural for VLM visual tokens, but it does not transfer directly to pure LLM settings, where the objects being pruned may be discrete text tokens, token-indexed KV entries, or weight structures whose semantics are not preserved under the same feature-space perturbation.
DiffPrune further introduces training-time cost through the Scorer and Denoiser, even though the Throttler and Denoiser are removed at inference.
Finally, while our pilot studies show substantially higher gradient-direction coherence and smoother scorer loss landscapes than Gumbel-Softmax pruning, a complete theoretical account of how feature dimensionality, the geometry of the $K$-th order statistic, and surrogate-gradient bias interact remains open.

Despite these limitations, these properties also suggest promising directions for future work in high-resolution visual domains. In particular, medical imaging and related settings are characterized by substantial redundancy in visual tokens~\cite{yang2025one,young2026fewer,young2026scalar}. This phenomenon has been widely observed in medical image analysis~\cite{yang2024segmentation,yang2023geometry,chen2026tc,wu2026multimodal,xu2026unified}, multimodal medical foundation models~\cite{xu2023learning,xu2024medvilam,xu2024foundation,feng2026efficient}, and visual prune tasks~\cite{he2026beyond,he2026autoselect,he2026stepwise,chen2026learnable,chen2026pathselect,gao2026zerosense}, suggesting that DiffPrune may be especially effective in these scenarios.

\end{document}